\documentclass[a4paper,UKenglish,cleveref, autoref, thm-restate]{lipics-v2021}

\pdfoutput=1 
\hideLIPIcs  

\usepackage{amsmath,amsfonts,amssymb,latexsym}
\usepackage{stmaryrd,proof}
\usepackage{hyperref}
\usepackage{multirow} 

\usepackage[table]{xcolor}
\usepackage{caption}
\usepackage{subcaption}
\usepackage{soul}
\usepackage{tikz}
\usetikzlibrary{calc} 
\usepackage{color}

\newcommand{\cf}{\emph{cf}.}
\DeclareMathOperator*{\argmax}{arg\,max}
\newcommand{\bool}[1]{[\![#1]\!]}
\newcommand{\norm}[2]{|\!|#2|\!|_{#1}}
\renewcommand{\vec}[1]{\mathbf{#1}}
\newcommand{\M}{\mathcal{M}} 

\newcommand{\hyp}{\mathcal{H}} 
\newcommand{\hypX}[1]{\vec{x}_{#1}} 
\renewcommand{\O}{\mathcal{O}}
\newcommand{\Path}{\mathbf{Path}}

\newcommand{\distr}[1]{\mathcal{D}(#1)}

\newcommand{\trans}[1]{\xrightarrow{#1}}

\newcommand{\prismcomment}[1]{\mbox{ \color{teal} \texttt{#1}}}
\newcommand{\prismkeyword}[1]{{\color{purple}\mathtt{#1}}}
\newcommand{\prismident}[1]{\mathtt{#1}}
\newcommand{\prismtab}{\hspace*{0.5cm}}
\newcommand{\prism}{\textsc{Prism}}

\newcounter{inlineenum}
\renewcommand{\theinlineenum}{\roman{inlineenum}}
\newenvironment{inlineenum}
  {\unskip\ignorespaces\setcounter{inlineenum}{0}%
   \renewcommand{\item}{\refstepcounter{inlineenum}{(\theinlineenum)~}}}
  {\ignorespacesafterend}

\title{MM Algorithms to Estimate Parameters in Continuous-time Markov Chains} 

\author{Giovanni Bacci}{Dept. of Computer Science, Aalborg University, Denmark}{giovbacci@cs.aau.dk}{https://orcid.org/0000-0001-8529-0681}{}
\author{Anna Ing{\'{o}}lfsd{\'{o}}ttir}{Dept. of Computer Science, Reykjav\'{i}k University, Iceland}{annai@ru.is}{}{}
\author{Kim G.\ Larsen}{Dept. of Computer Science, Aalborg University, Denmark}{kgl@cs.aau.dk}{}{VILLUM Investigator project S4OS, funded by the Velux Foundation.}
\author{Rapha\"{e}l Reynouard}{Dept. of Computer Science, Reykjav\'{i}k University, Iceland}{raphal20@ru.is}{}{}
\authorrunning{G.\ Bacci, A.\ Ing{\'{o}}lfsd{\'{o}}ttir, K.G.\ Larsen, and R.\ Reynouard} 
\Copyright{Giovanni Bacci, Anna Ing{\'{o}}lfsd{\'{o}}ttir, Kim G.\ Larsen, Rapha\"{e}l Reynouard} 



\relatedversion{full version hosted on arXiv} 

\funding{Anna Ing{\'{o}}lfsd{\'{o}}ttir and Rapha\"{e}l Reynouard: Learning and Applying Probabilistic Systems (Project nr.\ 206574-051) funded by the Icelandic Research Fund.}

\nolinenumbers 

\EventEditors{John Q. Open and Joan R. Access}
\EventNoEds{2}
\EventLongTitle{42nd Conference on Very Important Topics (CVIT 2016)}
\EventShortTitle{CVIT 2016}
\EventAcronym{CVIT}
\EventYear{2016}
\EventDate{December 24--27, 2016}
\EventLocation{Little Whinging, United Kingdom}
\EventLogo{}
\SeriesVolume{42}
\ArticleNo{23}

\begin{document}

\maketitle

\begin{abstract}

Continuous-time Markov chains (CTMCs) are popular modeling formalism that constitutes the underlying semantics for real-time probabilistic systems such as queuing networks, stochastic process algebras, and calculi for systems biology. 
\prism\ and \textsc{Storm} are popular model checking tools that provide a number of powerful analysis techniques for CTMCs. These tools accept models expressed as the parallel composition of a number of modules interacting with each other.  

The outcome of the analysis is strongly dependent on the parameter values used in the model which govern the timing and probability of events of the resulting CTMC. 
However, for some applications, parameter values have to be empirically estimated from partially-observable executions.

In this work, we address the problem of estimating parameter values of CTMCs expressed
as \prism\ models from a number of partially-observable executions. 
We introduce the class parametric CTMCs ---CTMCs where transition rates are polynomial functions over a set of parameters--- as an abstraction of CTMCs covering a large class of \prism\ models. 
Then, building on a theory of algorithms known by the initials MM, for minorization--maximization, we present iterative maximum likelihood estimation algorithms for parametric CTMCs covering two learning scenarios: when both state-labels and dwell times are observable, or just state-labels are. 

We conclude by illustrating the use of our technique in a simple but non-trivial case study: the analysis of the spread of COVID-19 in presence of lockdown countermeasures. 

\end{abstract}
\section{Introduction} \label{sec:intro} 
A continuous-time Markov chain (CTMC) is a model of a dynamical system that, upon entering some state, remains in that state for a random real-valued amount of time ---called the dwell time or sojourn time--- and then transitions probabilistically to another state. 
CTMCs are popular models in performance and dependability analysis. They have wide application and constitute the underlying semantics for real-time probabilistic systems such as queuing networks~\cite{LazowskaZGS84}, stochastic process algebras~\cite{Hillston94}, and calculi for systems biology~\cite{CiocchettaH09,KwiatkowskaNP08}. 

Model checking tools such as \prism~\cite{KwiatkowskaNP11} and \textsc{Storm}~\cite{DehnertJK017}
provide access to a number of powerful analysis techniques for CTMCs. 
Both tools accept models written in the \prism\ language, a state-based language based on~\cite{AlurH99b} that represents synchronous and asynchronous components in a uniform framework that supports compositional design. 
\begin{figure}[t!]
    \centering
    \begin{subfigure}[t]{0.5\textwidth}
        \centering
        \input{prismmodels/SIRuno.tex}
    \end{subfigure}%
    ~
    \begin{subfigure}[t]{0.45\textwidth}
        \centering
        \input{prismmodels/SIRagents.tex}
    \end{subfigure}
    \caption{(Left) SIR model with lockdown from~\cite{Milazzo21}, (Right) Semantically equivalent formulation of the model to the left where susceptible, infected, and recovered individuals are modeled as distinct modules interacting with each other via synchronization.} \label{fig:bigSIR}
\end{figure}
For example, consider the two semantically equivalent \prism\ models depicted in Fig.~\ref{fig:bigSIR} implementing a variant of the Susceptible-Infected-Recovered (SIR) model proposed in~\cite{Milazzo21} to describe the spread of disease in presence of lockdown restrictions. The model depicted to the left consists of a single module, whereas the one to the right implements a compositional design where modules interact by synchronizing on two actions: \texttt{infection} and \texttt{recovery}.

Both models distinguish between three types of individuals: susceptible, infected, and recovered. Susceptible individuals become infected through contact with another infected person and can recover without outside interference. The SIR model is parametric in $\mathtt{beta}$, $\mathtt{gamma}$, and $\mathtt{plock}$. $\mathtt{beta}$ is the \emph{infection coefficient}, describing the probability of infection after the contact of a susceptible individual with an infected one; $\mathtt{gamma}$ is the \emph{recovery coefficient}, describing the rate of recovery of an infected individual (in other words, $1/\mathtt{gamma}$ is the time one individual requires to recover); and $\mathtt{plock} \in [0,1]$ is used to scale down the infection coefficient modeling restrictions to reduce the spread of disease. 

Clearly, the outcome of the analysis of the above SIR model is strongly dependent on the parameter values used in each module, as they govern the timing and probability of events of the CTMC describing its semantics. However, in some application domains, parameter values have to be empirically evaluated from a number of partially-observable executions of the model. To the best of our knowledge, neither \prism\ nor \textsc{Storm} provide integrated support for this task, leaving the burden of estimating parameter values to the modeler. A paradigmatic example is the modeling pipeline described in~\cite{Milazzo21}, where the parameters of the SIR model in Fig.~\ref{fig:bigSIR} are estimated based on a definition of the model as ODEs, and later used in an approximation of the original SIR model designed to reduce the state space of the SIR model in Fig.~\ref{fig:bigSIR} (left).
Such modeling pipelines require high technical skills, are error-prone, and are time-consuming, thus limiting the applicability and the user base of model checking tools. 

In this work, we address the problem of estimating parameter values of CTMCs expressed as \prism\ models from a number of partially-observable executions. 
The expressive power of the \prism\ language brings two technical challenges:
\begin{inlineenum}
    \item the classic state-space explosion problem due to modular specification, and
    \item the fact that the transition rates of the CTMCs result from the algebraic composition of the rates of different (parallel) modules which are themselves defined as arithmetic expressions over the parameters (\emph{cf}.\ Fig.~\ref{fig:bigSIR}).
\end{inlineenum}
We address the second aspect of the problem by considering a class of \emph{parametric} CTMCs, which are CTMCs where transition rates are polynomial functions over a fixed set of parameters. In this respect, parametric CTMCs have the advantage to cover a rich subclass of \prism\ models and to be closed under the operation of parallel composition implemented by the \prism\ language. 

Following the standard approach, we pursue the maximum likelihood estimate (MLE), i.e., we look for the parameter values that achieve the maximum joint likelihood of the observed execution sequences. However, given the non-convex nature of the likelihood surface, computing the global maximum that defines the MLE is computationally intractable~\cite{Terwijn02}.

To deal with this issue we employ a theoretical iterative optimization principle known as MM algorithm~\cite{Lange16,Lange13}. 
The well-known EM algorithm~\cite{Dempster77} is an instance of MM optimization framework and is a versatile tool for constructing optimization algorithms. 
MM algorithms are typically easy to design, numerically stable, and in some cases amenable to accelerations~\cite{JamshidianJ97,ZhouAL11}. 
The versatility of the MM principle consists in the fact that is built upon a simple theory of inequalities, allowing one to derive optimization procedures. In fact, these procedures appear to be much easier than the derivation of a corresponding EM algorithm that relies on choosing appropriate missing data structures, i.e., latent variables. 
As the EM algorithm, the MM principle is useful to derive iterative procedures for maximum likelihood estimation which increase the likelihood at each iteration and converge to some local optimum.  

The main technical contribution of the paper consists in laying out MM techniques for devising novel iterative maximum likelihood estimation algorithms for parametric CTMCs covering two learning scenarios.

In the first scenario, we assume that state labels and dwell times are observable variables while state variables are hidden. The learning procedure devised for this case is a generalization of the Baum-Welch 
algorithm~\cite{Rabiner89} ---an EM algorithm that estimates transition probabilities in hidden Markov models--- to parametric CTMCs. 

In the second scenario, state labels are observable while state variables and dwell time variables are hidden. In contrast with the first case, the objective function that defines the MLE achieves the same value on all the CTMCs sharing the same embedded Markov chain. Thus, a standard adaptation of the Baum-Welch algorithm to this case would not lead to a procedure able to learn the continuous-time aspects of the observed system. Nevertheless, by making an analogy between the way transitions ``compete'' with each other in race conditions and the Bradley--Terry model of ranking~\cite{BradleyT52}, we successfully extend the solution devised for the first scenario with techniques used by Lange, Hunter, and Yang in \cite{LangeHY00} for finding rank estimates in the Bradley--Terry model. 
We provide experimental evidence that, when the model has sufficiently many constant transition rates, our solution effectively converge to the true parameter values of the model by hinging on the rate values that are known in the model. Note that this condition is easily fulfilled when one of the components is fully observable. A typical example is the model a microcontroller component running within a partially observable physical environment. Other examples may arise from website analysis for reviewing a website's performance w.r.t.\ user experience. 

We demonstrate the effectiveness of our estimation procedure on a case study taken from~\cite{Milazzo21}: the analysis of the spread of COVID-19 in presence of lockdown countermeasures. In particular, we showcase how our technique can be used to simplify modeling pipelines that involve a number of modifications of the model ---possibly introducing approximations--- and the re-estimation of its parameters. 

\paragraph*{Related Work} 
In~\cite{GeorgoulasHS18,GeorgoulasHMS14} Georgoulas et al.\ employ probabilistic programming to implement a variant of Bio-PEPA~\cite{CiocchettaH09} called ProPPA. ProPPA is a stochastic process algebra with inference capabilities that allows some rates to be assigned a prior distribution, capturing the modeler’s belief about the likely values of the rates. Using probabilistic inference, the ProPPA model may be combined with the observations to derive updated probability distributions over rates. 

Before ProPPA, Geisweiller proposed EMPEPA~\cite{Geisweiller2006}, an EM algorithm that estimates the rate values inside a PEPA model.

A closely related work is~\cite{WeiWT02} where they learn continuous-time hidden Markov models to do performance evaluation. There, observations are regarded as (discrete-time) periodic observations with fixed period $\Delta$. The learning method works in two steps: first, they employ the Baum-Welch algorithm~\cite{Rabiner89} to estimate the transition probability matrix of a hidden Markov model, then they obtain the infinitesimal generator of the CTMC from the learned transition probability matrix. In contrast with~\cite{WeiWT02}, we are able to derive a simpler procedure that directly extends the Baum-Welch algorithm to parametric CTMCs. 

In~\cite{SenVA04}, Sen et al.\ present an algorithm based on the state merging paradigm of \textsc{Alergia}~\cite{CarrascoO94} to learn a CTMC from timed observations. In contrast with our work, \cite{SenVA04} does not perform parameter estimation over structured models, but learns an unstructured (transition-labeled) CTMC. 

Another related line of research is parameter synthesis of Markov models~\cite{JansenJK22}.
In contrast with our work, parameter synthesis revolves around the problem of finding (some or all) parameter instantiations of the model that satisfy a given logical specification. 



\section{Preliminaries and Notation} \label{sec:prelim}
We denote by $\mathbb{R}$, $\mathbb{Q}$, and $\mathbb{N}$ respectively the sets of real, rational, and natural numbers, 
and by $\Sigma^n$, $\Sigma^*$ and, $\Sigma^\omega$ respectively the set of words of length $n \in \mathbb{N}$, finite length, and infinite length, built over the finite alphabet $\Sigma$.

We use  $\distr{\Omega}$ to denote the set of discrete probability distributions on $\Omega$, i.e., functions $\mu \colon \Omega \to [0,1]$, such that $\mu(X) = 1$, where $\mu(E) = \sum_{x \in E} \mu(x)$ for $E \subseteq X$. 
For a proposition $p$, we write $\bool{p}$ for the Iverson bracket of $p$, i.e., $\bool{p} = 1$ if $p$ is true, otherwise $0$. 

A labelled continuous-time Markov chain (CTMC) is defined as follows.
\begin{definition} A labelled CTMC is a tuple $\M = (S, R, \pi, \ell)$ where $S$ is a finite set of states, $R \colon S \times S \to \mathbb{R}_{\geq 0}$ is the transition rate function, $\pi \in \distr{S}$ the initial distribution of states, and $\ell \colon S \to L$ is a labelling function which assigns to each state an observable label $\ell(s)$.
\end{definition}

The transition rate function assigns rates $r = R(s, s')$ to each pair of states $s,s' \in S$ which are to be seen as transitions of the form $s \trans{r} s'$. A transition $s \trans{r} s'$ can only occur if $r > 0$. In this case, the probability of this transition to be triggered within $\tau \in \mathbb{R}_{>0}$ time-units is $1 - e^{- r \, \tau}$. When, from a state $s$, there are more than one outgoing transition with positive rate, we are in presence of a \emph{race condition}. In this case, the first transition to be triggered determines which label is observed as well as the next state of the CTMC. According to these dynamics, the time spent in state $s$ before any transition occurs, called \emph{dwell time}, is exponentially distributed with parameter $E(s) = \sum_{s' \in S} R(s,s')$, called \emph{exit-rate} of $s$. A state $s$ is called \emph{absorbing} if $E(s) = 0$, that is, $s$ has no outgoing transition. Accordingly, when the CTMC ends in an absorbing state it will remain in the same state indefinitely. 
The probability that the transition $s \trans{r} s'$ is triggered from $s$ is $r / E(s)$ and is independent from the time at which it occurs.
Accordingly, from the CTMC $\M$, we construct a (labelled) discrete-time Markov chain $\mathit{emb}(\M) = (S, P,\pi, \ell)$ with transition probability function $P \colon S \times S \to [0,1]$ defined as
\begin{equation*}
P(s, s') = \begin{cases}
R(s,s') / E(s) & \text{if $E(s) \neq 0$} \\
1 & \text{if $E(s) = 0$ and $s = s'$} \\
0 & \text{otherwise}
\end{cases}
\end{equation*}

\begin{remark}
A CTMC can be equivalently described as a tuple $(S,{\to},s_0,\ell)$ where ${\to} \subseteq S \times \mathbb{R}_{\geq 0} \times S$ is a transition \emph{relation}. The transition rate function $R$ induced by $\to$ is obtained as, $R(s,s') = \sum \{ r \mid s \trans{r} s' \}$ for arbitrary $s,s' \in S$.
\end{remark}


An \emph{infinite path} of a CTMC $\M$ is a sequence $s_0 \tau_0 s_1 \tau_1 s_2 \tau_2 \cdots \in (S \times \mathbb{R}_{>0})^\omega$ where $R(s_i, s_{i+1}) > 0$ for all $i \in \mathbb{N}$. A \emph{finite path} is a sequence $s_0 \tau_0 \cdots s_{k-1} \tau_{k-1} s_k$ where $R(s_i, s_{i+1}) > 0$ and $\tau_i \in \mathbb{R}_{>0}$ for all $i \in \{1, \dots, k-1\}$ and $s_k$ is absorbing. The meaning of a path is that the system started in state $s_0$, where it stayed for time $\tau_0$, then transitioned to state $s_1$ where it stayed for time $\tau_1$, and so on. For a finite path the system eventually reaches an absorbing state $s_k$, where it remains. 
We denote by $\Path_\M$ the set of all (infinite and finite) paths of $\M$. The formal definition of the probability space over $\Path_\M$ induced by $\M$ can be given by following the classical cylinder set construction (see e.g., ~\cite{BaierHHK03,KwiatkowskaNP07}).


Finally, we define the random variables $S_i$, $L_i$, and $T_i$ ($i \in \mathbb{N}$) that respectively indicate the $i$-th state, its label, and $i$-th dwell time of a path. 

\paragraph*{The MM Algorithm} \label{sec:MMintro}
The MM algorithm is an iterative optimisation method. The acronym MM has a double interpretation: in minimization problems, the first M stands for majorize and the second for minorize; dually, in maximization problems, the first M stands for minorize and the second for maximize. 
In this paper we only focus on maximizing an objective function $f(\vec{x})$, hence we tailor 
the presentation of the general principles of the MM framework to maximization problems.
The MM algorithm is based  on the concept of \emph{surrogate function}. A surrogate function $g(\vec{x} \mid \vec{x}_m)$ is said to \emph{minorize} a function $f(\vec{x})$ at $\vec{x}_m$ if 
\begin{align}
&f(\vec{x}_m) = g(\vec{x}_m \mid \vec{x}_m) \,,   \label{eq:tangent} \\ 
&f(\vec{x}) \geq g(\vec{x} \mid \vec{x}_m)\, \quad \text{for all } \vec{x} \neq \vec{x}_m \,. \label{eq:minorization}
\end{align}  
In the maximization variant of the MM algorithm, we maximize the surrogate minorizing function $g(\vec{x} \mid \vec{x}_m)$ rather than the actual function $f(\vec{x})$. If $\vec{x}_{m+1}$ denotes the maximum of the surrogate $g(\vec{x} \mid \vec{x}_m)$, then we can show that the next iterate $\vec{x}_{m+1}$ forces $f(\vec{x})$ uphill, Indeed, the inequalities
\begin{equation*}
f(\vec{x}_m) = g(\vec{x}_m \mid \vec{x}_m) \leq g(\vec{x}_{m+1} \mid \vec{x}_m) \leq f(\vec{x}_{m+1}) 
\end{equation*}
follow directly from the definition of $\vec{x}_{m+1}$ and the axioms \eqref{eq:tangent} and \eqref{eq:minorization}.

The art in devising an MM algorithm revolves around intelligent choices of minorizing functions. This work relies on three inequalities. 
The first basic minorization builds upon Jensen's inequality. For $x_i > 0$, $y_i > 0$ ($i =1 \dots n$), 
\begin{equation}
\ln \left(  \sum_{i = 1}^{n} x_i \right) \geq \sum_{i = 1}^{n} \frac{y_i}{\sum_{j = 1}^n y_j} \ln \left(  \frac{\sum_{j = 1}^n y_j}{y_i} x_i\right) \label{eq:basicmin1}
\end{equation}
Note that the above inequality becomes an equality whenever $x_i = y_i$ for all $i =1 \dots n$. Remarkably, the EM algorithm~\cite{Dempster77} is a special case of the MM algorithm which revolves around the above basic minorization when additionally the values $x_i$ and $y_i$ describe a probability distribution, i.e., $\sum_{i = 1}^{n} x_i = 1$ and $\sum_{i = 1}^{n} y_i = 1$.

Our second basic minorization derives from the strict concavity of the logarithm function, which implies for $x,y > 0$ that 
\begin{equation}
- \ln x \geq 1 - \ln y -  x/y\label{eq:basicmin2} 
\end{equation}
with equality if and only if $x = y$. Note that the above inequality restates the supporting hyperplane property of the convex function $- \ln x$.

The third basic minorization~\cite[\S 8.3]{Lange13} derives from the generalized arithmetic-geometric mean inequality 
which implies, for positive $x_i$, $y_i$, and $\alpha_i$ and $\alpha = \sum_{i = 1}^{n} \alpha_i$, that
\begin{equation}
    - \prod_{i = 1}^{n} x_i^{\alpha_i} 
    \geq - \left( \prod_{i = 1}^{n} y_i^{\alpha_i} \right) \sum_{i = 1}^n \frac{\alpha_i}{\alpha} \left( \frac{x_i}{y_i}\right)^{\alpha}\,.
    \label{eq:basicmin3}
\end{equation}
Note again that equality holds when all $x_i = y_i$.

Because piecemeal composition of minorization works well, our derivations apply the above basic minorizations to strategic parts of the objective function, leaving other parts untouched. 
Finally, another aspect that can simplify the derivation of MM algorithms comes from the fact that the iterative maximization procedure hinges on finding $\vec{x}_{m+1} = \argmax_{\vec{x}} g(\vec{x} \mid \vec{x}_m)$. Therefore, we can equivalently use any other surrogate function $g'(\vec{x} \mid \vec{x}_m)$ satisfying $\argmax_{\vec{x}} g(\vec{x} \mid \vec{x}_m) = \argmax_{\vec{x}} g'(\vec{x} \mid \vec{x}_m)$. 
This is for instance the case when $g(\vec{x} \mid \vec{x}_m)$ and $g'(\vec{x} \mid \vec{x}_m)$ are equal up to some (irrelevant) constant $c$, that is $g(\vec{x} \mid \vec{x}_m) = g'(\vec{x} \mid \vec{x}_m) + c$.

\section{Parametric Continuous-time Markov chains}
As mentioned in the introduction, the \prism\ language offers constructs for the modular design of CTMCs within a uniform framework that represents synchronous and asynchronous module interaction. 
For example, consider the \prism\ models depicted in Fig.~\ref{fig:bigSIR}. The behavior of each module is described by a set of commands which take the form 
\mbox{$[\prismident{action}] \; \prismident{guard} \; \rightarrow \prismident{rate} \colon \prismident{update}$} representing a set of transitions of the module. The guard is a predicate over the state variables in the model. The update and the rate describe a transition that the module can make if the guard is true. 
The command optionally includes an action used to force two or more modules to make transitions simultaneously (i.e., to synchronize). For example, in the left model in Fig.~\ref{fig:bigSIR}, in state $(50, 20, 5)$ (i.e., $s = 50$, $i = 20$, and $r = 5$), the composed model can move to state $(49,21,5)$ by synchronizing over the action $\mathtt{infection}$. The rate of this transition is equal to the product of the individual rates of each module participating in an $\mathtt{infection}$ transition, which in this case amounts to $0.01 \cdot \mathtt{beta} \cdot \mathtt{plock}$.
Commands that do not have an action represent asynchronous transitions that can be taken independently (i.e., asynchronously) from other modules. 

By default, all modules are combined following standard parallel composition in the sense of the parallel operator from Communicating Sequential Processes algebra (CPS), that is, modules synchronize over all their common actions. The \prism\ language offers also other CPS-based operators to specify the way in which modules are composed in parallel.

Therefore, a parametric representation of a CTMC described by a \prism\ model shall consider \emph{transition rate expressions} which are closed under finite sums and finite products: sums deal with commands with overlapping guards and updates, while products take into account synchronization.

Let $\vec{x} = (x_1,\dots, x_n)$ be a vector of parameters. We write $\mathcal{E}$ for the set of polynomial maps $f \colon \mathbb{R}_{\geq 0}^n \to \mathbb{R}_{\geq 0}$ of the form $f(\vec{x}) = \sum_{i = 1}^{m} b_i \prod_{j=1}^{n} x_j^{a_{ij}}$, 
where $b_i \in \mathbb{R}_{\geq 0}$ and $a_{ij} \in \mathbb{N}$ for $i \in \{ 1, \dots, m \}$ and $j \in \{ 1, \dots, n \}$.
Note that $\mathcal{E}$ is a commutative semiring satisfying the requirements established above for transition rate expressions. 

We are now ready to introduce the notion of \emph{parametric} continuous-time Markov chain. 
\begin{definition}
A parametric CTMC is a tuple $\mathcal{P} = (S,R,s_0, \ell)$ where $S$, $s_0$, and $\ell$ are defined as for CTMCs, and $R \colon S \times S \to \mathcal{E}$ is a parametric transition rate function.
\end{definition}
Intuitively, a parametric CTMC $\mathcal{P} = (S,R,s_0, \ell)$ defines a family of CTMCs arising by plugging in concrete values for the parameters $\vec{x}$. Given a parameter evaluation $\vec{v} \in \mathbb{R}_{\geq 0}^n$, we denote by $\mathcal{P}(\vec{v})$ the CTMC associated with $\vec{v}$, and $R(\vec{v})$ for its rate transition function. Note that by construction $R(\vec{v})(s,s') \geq 0$ for all $s,s' \in S$, therefore $\mathcal{P}(\vec{v})$ is a proper CTMC.

As for CTMCs, parametric transitions rate functions can be equivalently described by means of a transition relation ${\to} \subseteq S \times \mathcal{E} \times S$, where the parametric transition rate from $s$ to $s'$ is $R(s,s')(\vec{x}) = \sum \{f(\vec{x}) \mid s \trans{f} s' \}$. 

\begin{example}
Consider the SIR model in Fig.~\ref{fig:bigSIR} with parameters $\mathtt{beta}$, $\mathtt{gamma}$, and $\mathtt{plock}$. The semantics of this model is a parametric CTMC with states $S=\{(s,i,r) \mid s,i,r \in \{0, \dots, 10^5 \} \}$ and initial state $(99936, 48, 16)$. 
For example, the initial state has two outgoing transitions: one that goes to $(99935, 49, 16)$ with rate $48.96815 \cdot \mathtt{beta} \cdot \mathtt{plock}$, and the other that goes to $(99935,48,17)$ with rate $49 \cdot \mathtt{gamma} \cdot \mathtt{plock}$. \qed
\end{example}
 
One relevant aspect of the class of parametric CTMCs is the fact that it is closed under parallel composition in the sense described above. As a consequence, the study of parameter estimation of \prism\ models from observed data can be conveniently addressed as maximum likelihood estimation for parametric CTMCs.


\section{Learning Parameters from Observed Sample Data} \label{sec:learnCTMC}
In this section we present two algorithms to estimate the parameters of parametric CTMC $\mathcal{P}$ from a collection of i.i.d.\ observation sequences $\O = \vec{o}_1, \dots, \vec{o}_J$.
The two algorithms consider two different types of observations: timed and non-timed. A \emph{timed observation} $\ell_{0:k},\tau_{0:k-1}$ is a finite sequence $\ell_0 \tau_0 \cdots \tau_{k-1} \ell_k$ representing consecutive dwell times and state labels observed during a random execution of $\M$. Similarly, a \emph{non-timed observation} $\ell_{0:k}$ represents a sequence of consecutive state labels observed during a random execution of $\M$. 
Both algorithms follow a maximum likelihood approach: the parameters $\vec{x}$ are estimated to maximize the joint likelihood $\mathcal{L}(\mathcal{P}(\vec{x}) | \O)$ of the observed data. 
When $\mathcal{P}$ and $\O$ are clear from the context, we simply write $\mathcal{L}(\vec{x})$ for the joint likelihood. 

Hereafter we present a solution to the maximum likelihood estimation problem building on an optimization framework known by the name MM algorithm~\cite{Lange13,Lange16}. 
In this line, our algorithms start with an initial hypothesis $\hypX{0}$ and iteratively improve the current hypothesis $\hypX{m}$, in the sense that the likelihood associated with the next hypothesis $\hypX{m+1}$ enjoys the inequality $\mathcal{L}(\hypX{m}) \leq \mathcal{L}(\hypX{m+1})$. The procedure terminates when the improvement does not exceed a fixed threshold $\epsilon$, namely when $\mathcal{L}(\hypX{m}) - \mathcal{L}(\hypX{m-1}) \leq \epsilon$. 

\subsection{Learning from Timed Observations} \label{sec:learntimed}
Assume we have $J$ i.i.d.\ timed observation sequences $\O = \vec{o}_1, \dots, \vec{o}_J$ where $\vec{o}_j = \ell^j_{0:k_j}, \tau^j_{1:k_j-1}$ ($j=1\dots J$). We want to estimate a valuation of the parameters $\vec{x}$ of $\mathcal{P}$ that maximises the joint likelihood function $\mathcal{L}(\vec{x}) = \prod_{j=1}^{J} l(\vec{o}_j | \mathcal{P}(\vec{x}))$ where the likelihood of an observation $\vec{o} = \ell_{0:k}, \tau_{0:k-1}$ for a generic CTMC $\M$ is
\begin{align}
&l (\vec{o} | \M) = \textstyle \sum_{s_{0:k}} l(S_{0:k} = s_{0:k}, L_{0:k} = \ell_{0:k}, T_{0:k-1} = \tau_{0:k-1} | \M) \notag \\
&\quad = \textstyle \sum_{s_{0:k}} P[S_{0:k} = s_{0:k}, L_{0:k} = \ell_{0:k} | \M] \cdot l(S_{0:k} = s_{0:k}, T_{0:k-1} = \tau_{0:k-1} | \M) \notag\\
&\quad = \textstyle \sum_{s_{0:k}} \bool{\ell(s_{0:k}) {=} \ell_{i:k}} \big( \prod_{i = 0}^{k-1} R(s_{i}, s_{i+1}) / E(s_{i}) \big) \, \big(\prod_{i = 0}^{k-1}E(s_{i})  \, e^{- E(s_{i}) \tau_{i}} \big) \notag\\
& \quad = \textstyle \sum_{s_{0:k}} \bool{\ell(s_{0:k}) = \ell_{i:k}} \prod_{i = 0}^{k-1} R(s_{i}, s_{i+1}) \cdot e^{- E(s_{i}) \tau_{i}} \,.\label{eq:LTObs2}
\end{align}

Before presenting an MM algorithm to solve the MLE problem above, we find it convenient to introduce some notation. Let $\mathcal{P} = (S,\to, s_0, \ell)$, we write $f_\rho$ for the rate map of the transition $\rho \in {\to}$, and write $s \to \cdot$ for the set of transitions departing from $s \in S$. 

Without loss of generality, we assume that the rate function $f_\rho$ of a transition is either a constant map, i.e., $f_\rho(\vec{x}) = c_r$ for some $c_r \geq 0$ or a map of the form $f_\rho(\vec{x}) = c_\rho \prod_{i = 1}^{n} x_i^{a_{\rho i}}$ for some $c_\rho > 0$ and $a_{\rho i} > 0$ for some $i \in \{1,\dots, n\}$; we write $a_\rho$ for $\sum_{i = 1}^n a_{\rho i}$. We denote by $\trans{c}$ the subset of transitions with constant rate function and $\trans{\vec{x}}$ for the remaining transitions. 

To maximize $\mathcal{L}(\vec{x})$ we propose to employ an MM algorithm based on the following surrogate function $g(\vec{x} | \hypX{m}) = \sum_{i = 1}^n g(x_i | \hypX{m})$ where
\begin{equation}
g(x_i | \hypX{m}) = 
    \sum_{\rho \in {\trans{\vec{x}}}} \xi_{\rho} a_{\rho i} \ln x_i 
    - \sum_{s} \sum_{\rho \in s \trans{\vec{x}} \cdot} \frac{f_\rho(\hypX{m}) a_{\rho i} \gamma_s}{a_\rho (x_{m i})^{a_\rho}} x_i^{a_\rho} \label{eq:surrogate-timed}
\end{equation} 
Here $\gamma_s = \sum_{j = 1}^{J} \sum_{t = 0}^{k_j-1} \gamma^j_s(t) \tau^j_{t}$ and $\xi_\rho = \sum_{j = 1}^{J} \sum_{t = 0}^{k_j-1} \xi^j_{\rho}(t)$, where $\gamma^j_{s}(t)$ denotes the likelihood that having observed $\vec{o}_j$ on a random execution of $\mathcal{P}(\hypX{m})$ the state $S_t = s$; 
and $\xi^j_{\rho}(t)$ is the likelihood that for such random execution the transition performed from state $S_t$ is $\rho$.

The following theorem states that the surrogate function $g(\vec{x} | \hypX{m})$ is a minorizer of the log-likelihood relative to the observed dataset $\O$. 
\begin{theorem} \label{thm:gMinorizes}
The surrogate function $g(\vec{x} | \hypX{m})$ minorizes $\ln \mathcal{L}(\vec{x})$ at $\hypX{m}$ up to an irrelevant constant. 
\end{theorem}
By Theorem~\ref{thm:gMinorizes} and the fact that the logarithm is an increasing function, we obtain that the parameter valuation that achieves the maximum of $g(\vec{x} | \hypX{m})$ improves the current hypothesis $\hypX{m}$ relative to likelihood function $\mathcal{L}(\vec{x})$.
\begin{corollary}
Let $\hypX{m+1} =\argmax_{\vec{x}} g(\vec{x} | \hypX{m})$, then $\mathcal{L}(\hypX{m}) \leq \mathcal{L}(\hypX{m+1})$.
\end{corollary}

The surrogate function $g(\vec{x} | \hypX{m})$ is easier to maximize than $\mathcal{L}(\vec{x})$ because its parameters are separated. Indeed, maximization of $g(\vec{x} | \hypX{m})$ is done by point-wise maximization of each univariate function $g(x_i | \hypX{m})$.
 This has two main advantages: first, it is easier to handle high-dimensional problems~\cite{Lange13,Lange16}; second, if one can choose to fix the value of some parameters, say $I \subset \{1\dots n\}$ and the maximization of $g(\vec{x} | \hypX{m})$ can be performed by maximizing $g(x_i | \hypX{m})$ for each $i \notin I$. 

%
 The maxima of $g(x_i | \hypX{m})$ are found among the \emph{non-negative} roots\footnote{Note that $P_i$ always admits non-negative roots. 
Indeed, $P_i(0) \leq 0$ and $P_i(M) > 0$ for $M > 0$ sufficiently large. Therefore, by the intermediate value theorem, there exists $y_0 \in [0,M)$ such that $P_i(y_0) = 0$.} of the polynomial function $P_i \colon \mathbb{R} \to \mathbb{R}$
\begin{equation}
P_i(y) = \sum_{s} \sum_{\rho \in s \trans{\vec{x}}} \frac{f_\rho(\hypX{m}) a_{\rho i} \gamma_s}{(x_{m i})^{a_\rho} } y^{a_\rho} - \sum_{\rho \in {\trans{\vec{x}}}} \xi_{\rho} a_{\rho i}
 \label{updateTimed}
\end{equation}

\begin{remark}
\label{rm:category1}
There are some cases when \eqref{updateTimed} admits a closed-form solution. For instance, when the parameter index $i$ satisfies the property $\forall \rho \in {\trans{\vec{x}}} .\,a_{\rho i} > 0 \implies a_{\rho} = C$ for some constant $C \in \mathbb{N}$, then maximization of $g(x_i | \hypX{m})$ leads to the following update
\begin{equation*}
    x_{(m+1) i} = \left[ \frac{ (x_{m i})^C \sum_{\rho \in {\trans{\vec{x}}}} \xi_{\rho} a_{\rho i} }{ \sum_{s} \sum_{\rho \in s \trans{\vec{x}}} f_\rho(\hypX{m}) a_{\rho i} \gamma_s } \right]^{1/C}
\end{equation*}
A classic situation when the above condition is fulfilled occurs when all transitions $\rho$ where $x_i$ appear (i.e., $a_{\rho i} > 0$), the transition rate is $f_\rho(\vec{x}) = c_\rho x_i$ (i.e., $a_{\rho i} = a_\rho = 1$). In that case, the above equation simplifies to
\begin{equation*}
    x_{(m+1) i} = \frac{ \sum_{\rho \in {\trans{\vec{x}}}} \xi_{\rho} }{ \sum_{s} \sum_{\rho \in s \trans{\vec{x}}} c_\rho \gamma_s }
\end{equation*}

For example, the parametric CTMC associated with the SIR models in Fig.~\ref{fig:bigSIR} satisfies the former property for all parameters, because all transition rates are expressions either of the form $c \cdot \mathtt{plock} \cdot \mathtt{beta}$ or the form $c \cdot \mathtt{plock} \cdot \mathtt{gamma}$ for some constant $c > 0$. Furthermore, if we fix the value of the parameter $\mathtt{plock}$ the remaining parameters satisfy the latter property. 
In Section~\ref{sec:case-study}, we will take advantage of this fact for our calculations.  \qed
\end{remark}

To complete the picture, we show how to compute the coefficients $\gamma^j_s(t)$ and $\xi^j_{\rho}(t)$. To this end, we employ standard forward and backward procedures.
We define the forward function $\alpha_s^j(t)$ and the backward function $\beta_s^j(t)$ respectively as
\begin{align*}
\alpha_s^j(t) &= l(L_{0: t} = \ell^{j}_{0: t},T_{0: t} = \tau^{j}_{0: t}, S_{t} = s | \mathcal{P}(\hypX{m})) \, \text{, and} \\
\beta_s^j(t) &= l(L_{t+1:k_j} = \ell^{j}_{t+1 :k_j}, T_{t+1:k_j-1} = \tau^{j}_{t+1:k_j-1} | S_{t} = s , \mathcal{P}(\hypX{m}) ) \, .
\end{align*}
These can be computed using dynamic programming according to the following recurrences: let $\mathcal{P}(\hypX{m}) = (S,R,s_0, \ell)$, then
\begin{align}
\alpha_s^j(t) &= \begin{cases}
	\bool{s = s_0} \, \omega^j_s(t) &\text{if $ t=0$} \\
	\omega^j_{s}(t) \sum_{s' \in S} \frac{R(s', s)}{E(s')} \, \alpha_{s'}^j(t-1) &\text{if $0 < t \leq k_j$}
\end{cases} 
\\
\beta_s^j(t) &= \begin{cases}
	1 &\text{if $ t= k_j$} \\
	  \sum_{s' \in S} \frac{R(s, s')}{E(s)} \, \beta_{s'}^j(t+1) \, \omega^j_{s'}(t+1) &\text{if $0 \leq t < k_j$}
\end{cases}
\end{align}
where 
\begin{equation}
\omega^j_s(t) = 
\begin{cases}
\bool{\ell(s) = \ell^j_t} E(s) e^{- E(s) \tau^j_{t}} 
    &\text{if $0 \leq t < k_j$,} \\
\bool{\ell(s) = \ell^j_t} &\text{if $t = k_j$.}  
\end{cases}
\label{eq:omegaobs}
\end{equation}

Finally, for $s \in S$ and $\rho = (s \trans{f_\rho} s')$, $\gamma^j_{s}(t)$ and $\xi^j_{\rho}(t)$ are related to the forward and backward functions as follows
\begin{align}
\gamma^j_{s}(t) = \frac{\alpha_s^j(t) \, \beta_s^j(t) }{ \sum_{s' \in S} \alpha_{s'}^j(t) \, \beta_{s'}^j(t)}  \, , 
&&
\xi^j_{\rho}(t) = \frac{ \alpha_s^j(t) f_\rho(\hypX{m}) \, \omega^j_{s'}(t+1) \,\beta_{s'}^j(t+1) }{ E(s) \sum_{s'' \in S} \alpha_{s''}^j(t) \, \beta_{s''}^j(t)}  \,. \label{eq:defGammaXi}
\end{align}

\subsection{Learning from Non-timed Observations} \label{sec:learnnontimed}
Let now assume we have collected $J$ i.i.d.\ non-timed observation sequences $\O = \vec{o}_1, \dots, \vec{o}_J$ where $\vec{o}_j = \ell^j_{0:k_j}$ ($j=1 \dots J$). As done before, we want to maximize the joint likelihood function $\mathcal{L}(\vec{x}) = \prod_{j=1}^{J} l(\vec{o}_j | \mathcal{P}(\vec{x}))$ where the likelihood an arbitrary non-timed observation $\vec{o} = \ell_{0:k}$ relative to the CTMC $\M$ is
\begin{align}
l(\ell^j_{1:k} | \M ) &= \textstyle \sum_{s_{0:k}} P[S_{0:k} = s_{0:k}, L_{0:k} = \ell^j_{1:k} | \M] \label{eq:LObs1} \\ 
&= \textstyle \sum_{s_{0:k}} \bool{\ell(s_{0:k}) {=} \ell_{i:k}} \prod_{i = 0}^{k-1} R(s_{i}, s_{i+1}) / E(s_{i}) \,. \label{eq:LObs2}
\end{align}
Looking at the formula above, it is clear that whenever two CTMCs $\M_1$ and $\M_2$ have the same embedded Markov chain they will have also the same likelihood value, i.e.,  $\mathcal{L}(\M_1 | \O) = \mathcal{L}(\M_2 | \O)$. The fact that dwell time variables are not observable leaves us with an MLE objective that does not fully capture the continuous-time aspects of the model under estimation. 

A similar problem shows up also in the Bradley--Terry model of ranking~\cite{BradleyT52}. This model is intuitively understood via a sport analogy. Given a set of teams where each team $i$ is assigned a rank parameter $r_i > 0$, assuming that ties are not possible, team $i$ beats team $j$ with probability $r_i / (r_i + r_j)$. If this outcome occurs $c_{ij}$ times during a tournament, then the probability of the whole tournament is $L(\vec{r}) = \prod_{i,j} (r_i / (r_i + r_j))^{c_{ij}}$, assuming that games are independent one another. Clearly, $L(\vec{r}) = L(c \,\vec{r})$ for any $c > 0$. Under mild assumptions, the function $L(\vec{r})$ admits a unique maximum when the value of one rank, say $r_1$, is fixed a priori.  

Back to our problem, we claim that the race conditions among transitions can be interpreted under the Bradley--Terry model of ranking. As a consequence, when the number of parametric transitions is sufficiently small relative to that of constant transitions, the estimation of the unknown transition rates can hinge on the value of the transition rates that are fixed, leading the algorithm to converge to the real parameter values. 

For the non-timed maximum likelihood estimation problem we devise an MM algorithm based on the surrogate function $h(\vec{x} | \hypX{m}) = \sum_{i = 1}^{n} h(x_i | \hypX{m})$ for
\begin{equation}
h(x_i | \hypX{m}) =  
    \sum_{\rho \in {\trans{\vec{x}}}} \hat\xi_{\rho} a_{\rho i} \ln x_i 
    - \sum_{s} \sum_{\rho \in s \trans{\vec{x}} \cdot} \frac{f_\rho(\hypX{m}) \, a_{\rho i} \,  \hat\gamma_s}{E_{m}(s) \, a_\rho \, x_{m i}^{a_\rho}} x_i^{a_\rho} \label{eq:surrogate-nontimed}
\end{equation}  
where $E_{m}(s)$ denotes the exit rate of the state $s$ in $\mathcal{P}(\hypX{m})$, $\hat\gamma_s = \sum_{j = 1}^{J}\sum_{t = 0}^{k_j - 1} \hat\gamma_s^j(t)$, and $\hat\xi_\rho = \sum_{j = 1}^{J}\sum_{t = 0}^{k_j - 1} \hat\xi_\rho^j(t)$. 

This time, the coefficients $\hat{\gamma}^j_{s}(t)$ and $\hat{\xi}^j_{\rho}(t)$ denote respectively the probability that having observed $\vec{o}_j$ in a random execution of $\mathcal{P}(\hypX{m})$, the state $S_t$ is $s$, and the transition performed in state $S_t$ is $\rho$. $\hat{\gamma}^j_{s}(t)$ and $\hat{\xi}^j_{r}(t)$  can be computed using the same dynamic programming procedure described in Section~\ref{sec:learntimed} by replacing each occurrence of $\omega^j_s(t)$ with $\hat{\omega}^j_s(t) = \bool{\ell(s) = \ell^j_t}$.

The following theorem states that the surrogate function $h(\vec{x} | \hypX{m})$ is a minorizer of the log-likelihood relative to the observed (non-timed) dataset $\O$.
\begin{theorem} \label{thm:hMinorizes}
The surrogate function $h(\vec{x} | \hypX{m})$ minorizes $\ln \mathcal{L}(\vec{x})$ at $\hyp_{m}$ up to an irrelevant constant. 
\end{theorem}
\begin{proof}(sketch) To easy the presentation we assume that the parametric CTMC $\mathcal{P}$ under study has at most one transition between each pair of states. 
Starting from the log-likelihood function $\ln \mathcal{L}(\vec{x})$, we proceed with the following minorization steps\footnote{We denote by $f(\vec{x}) \cong f'(\vec{x})$ the fact that $f(\vec{x}) =  f'(\vec{x}) + C$ for some (irrelevent) constant $C$.}
\begin{align*}
&\ln \mathcal{L}(\vec{x}) = \sum_{j = 1}^{J} \ln l(\vec{o}_j | \mathcal{P}(\vec{x}))
= \sum_{j = 1}^{J}  \ln  \sum_{s_{0:k_j}} P[ s_{0:k_j}, \vec{o}_j | \mathcal{P}(\vec{x})]  \tag{by \eqref{eq:LObs1}} \\
&\geq \sum_{j = 1}^{J}  \sum_{s_{0:k_j}} {P[s_{0:k_j} | \vec{o}_j , \mathcal{P}(\hypX{m})]} \ln \left( \frac{P[ s_{0:k_j}, \vec{o}_j | \mathcal{P}(\vec{x})]}{P[s_{0:k_j} | \vec{o}_j , \mathcal{P}(\hypX{m})]} \right) \tag{by \eqref{eq:basicmin1}}\\
&\cong 
\sum_{j = 1}^{J}  \sum_{t = 1}^{k_j} \sum_{s_{0:k_j}} P[s_{0:k_j} | \vec{o}_j , \mathcal{P}(\hypX{m})] \big( \ln R(s_{t}, s_{t+1}) - \ln E(s_{t}) \big) \tag{by \eqref{eq:LObs2}} \\ 
&\cong \sum_{\rho \in {\trans{\vec{x}}}} \hat\xi_{\rho} \ln f_\rho(\vec{x}) + \sum_{s} \hat\gamma_s (- \ln E(s) )\tag{up-to const} \\
&\geq 
\sum_{i = 1}^{n} \sum_{\rho \in {\trans{\vec{x}}}} \hat\xi_{\rho} a_{\rho i} \ln x_i + \sum_{s} \hat\gamma_s \left( - \frac{E(s)}{E_{m}(s)} \right) \tag{by \eqref{eq:basicmin2}, up-to const} \\
& \geq \sum_{i = 1}^n \left[ 
    \sum_{\rho \in {\trans{\vec{x}}}} \hat{\xi}_{\rho} a_{\rho i} \ln x_i 
    - \sum_{s} \sum_{\rho \in s \trans{\vec{x}} \cdot} \frac{\hat{\gamma}_s f_\rho(\hypX{m}) a_{\rho i}}{E_{m}(s) a_\rho x_{m i}^{a_\rho}} x_i^{a_\rho} \right] \tag{$\triangle\triangle$} \label{eq:hatminEsbis} \\
& = h(\vec{x} | \hypX{m}) \tag{by \eqref{eq:surrogate-nontimed}}
\end{align*}
Where the step \eqref{eq:hatminEsbis} is justified by the minorization of $- E(s)$ obtained via \eqref{eq:basicmin3} as follows
\begin{equation*}
- E(s) 
\cong \sum_{\rho \in s \trans{\vec{x}} \cdot} c_{\rho} \left( - \prod_{i = 1}^{n} x_i^{a_{\rho i}} \right) 
\geq \sum_{\rho \in s \trans{\vec{x}}} - f_{\rho}(\hypX{m}) \sum_{i = 1}^{n} \frac{a_{\rho i}}{a_\rho} \left(\frac{x_i}{x_{m i}} \right)^{a_\rho} \,.
\end{equation*}
Hence, there exists a constant $C > 0$ such that $h(\vec{x} | \hypX{m}) + C$ minorizes $\ln \mathcal{L}(\vec{x})$ at $\hypX{m}$. 
\end{proof}

Notably, in the proof of Theorem~\ref{thm:hMinorizes} we employ the minorization~\eqref{eq:basicmin2} used in~\cite{LangeHY00} for finding rankings in the Bradley--Terry model.

As an immediate corollary of Theorem~\ref{thm:hMinorizes}, we have that the parameter valuations that maximize $h(\vec{x} | \hypX{m})$ improve the current hypothesis $\hypX{m}$ with respect to the ML objective.
\begin{corollary}
Let $\hypX{m+1} =\argmax_{\vec{x}} h(\vec{x} | \hypX{m})$, then $\mathcal{L}(\hypX{m}) \leq \mathcal{L}(\hypX{m+1})$.
\end{corollary}

As before, maximization of $h(\vec{x} | \hypX{m})$ is achieved by point-wise maximization of $h(x_i | \hypX{m})$. The maxima of $h(x_i | \hypX{m})$ are found among the \emph{non-negative} roots of the polynomial function
\begin{equation}
Q_i(y) = 
\sum_{s} \sum_{\rho \in s \trans{\vec{x}} \cdot} \frac{f_\rho(\hypX{m}) a_{\rho i} \hat\gamma_s}{E(s) (x_{m i})^{a_\rho}} y^{a_\rho}
    - \sum_{\rho \in {\trans{\vec{x}}}} \hat\xi_{\rho} a_{\rho i} 
     \label{updateUntimed}
\end{equation}
By arguments similar to those explained in Remark~\ref{rm:category1}, Equation~\eqref{updateUntimed} may admit a closed-form solution. 

\section{Experimental evaluation} \label{sec:experiments}
\begin{figure}[t]
\input{prismmodels/tandem.tex}
\caption{Prism model for the tandem queueing network from \cite{Hermanns99}.}
\label{model:tandem}
\end{figure}
We implemented the algorithms from Section~\ref{sec:learnCTMC} as an extention of the \texttt{Jajapy} Python library~\cite{jajapy}, which has the advantage of being compatible with \prism\ models.
In this section, we present an empirical evaluation of the efficiency of our algorithms as well as the quality of their outcome. 
To this end, we use the tandem queueing network model from \cite{Hermanns99} (\cf\ Fig.~\ref{model:tandem}) as a benchmark for our evaluation. 

The experiments have been designed according to the following setup. We assume that the state of $\mathtt{serverC}$ is fully observable ---i.e., its state variables $\mathtt{sc}$ and $\mathtt{ph}$ are--  as well as the size $\mathtt{c}$ of the queue and the value of $\mathtt{lambda}$. In contrast, we assume that the state of $\mathtt{serverM}$ is not observable. 

Each experiment consists in estimating the value of the parameters $\mathtt{mu1a}$, $\mathtt{mu1b}$, $\mathtt{mu2}$, and $\mathtt{kappa}$ from a training set consisting of $100$ observation sequences of length $30$, generated by simulating the \prism\ model depicted in Fig.~\ref{model:tandem}.
We perform this experiment both using timed and non-timed observations, by increasing the size $\mathtt{c}$ of the queue until the running time of the estimation exceeds a time-out set to $1$ hour. 
We repeat each experiment 10
times by randomly re-sampling the initial values of each unknown parameter $x_i$ in the interval $[0.1,\, 5.0]$. We annotate the running time as well as the relative error 
$\delta_i$ for each parameter $x_i$, calculated according to the formula $\delta_i = |e_i - r_i| / |r_i|$, where $e_i$ and $r_i$ are respectively the estimated value and the real value of $x_i$.

\begin{table}[!ht]
\begin{center}
\rowcolors{4}{gray!10}{gray!40}
\begin{tabular}{|c|c|c|c|c|c|c|c|c|}
\hline
\multirow{2}{*}{$\mathtt{c}$} & \multirow{2}{*}{$|S|$} & \multirow{2}{*}{$|{\to}|$} & \multicolumn{2}{c|}{Running time (s)} & \multicolumn{2}{c|}{$\norm{1}{\delta}$} & \multicolumn{2}{c|}{$\norm{\infty}{\delta}$}\\ \cline{4-9}
& & & Timed & Non-timed & Timed & Non-timed & Timed & Non-timed \\ \hline \hline
4 & 45 & 123 & 4.336 & 15.346 & 0.226 & 0.251 & 0.13 & 0.13\\ \hline
6 & 91 & 269 & 13.219 & 38.661 & 0.399 & 0.509 & 0.173 & 0.329\\ \hline
8 & 153 & 471 & 37.42 & 90.952 & 0.322 & 0.387 & 0.183 & 0.187\\ \hline
10 & 231 & 729 & 76.078 & 170.044 & 0.359 & 0.346 & 0.17 & 0.191\\ \hline
12 & 325 & 1043 & 160.694 & 276.383 & 0.343 & 0.616 & 0.165 & 0.289\\ \hline
14 & 435 & 1413 & 264.978 & 623.057 & 0.373 & 0.263 & 0.195 & 0.117\\ \hline
16 & 561 & 1839 & 458.766 & 774.642 & 0.406 & 0.427 & 0.245 & 0.192\\ \hline
18 & 703 & 2321 & 871.39 & 1134.037 & 0.249 & 0.783 & 0.14 & 0.49\\ \hline
20 & 861 & 2859 & 1425.65 & 1225.539 & 0.416 & 0.987 & 0.281 & 0.519\\ \hline
22 & 1035 & 3453 & 2031.587 & 1297.383 & 0.546 & 1.013 & 0.278 & 0.602\\ \hline
24 & 1225 & 4103 & 2675.794 & 1924.074 & 0.441 & 1.892 & 0.281 & 1.599\\ \hline
\end{tabular}
\end{center}
\caption{Comparison of the performance of the estimation for timed and non-timed observations on the tandem queueing network with different size of the queue.}
\label{tab:tandem-res}
\end{table}
\begin{figure}[ht]
    \centering
    \includegraphics[width=0.8\textwidth]{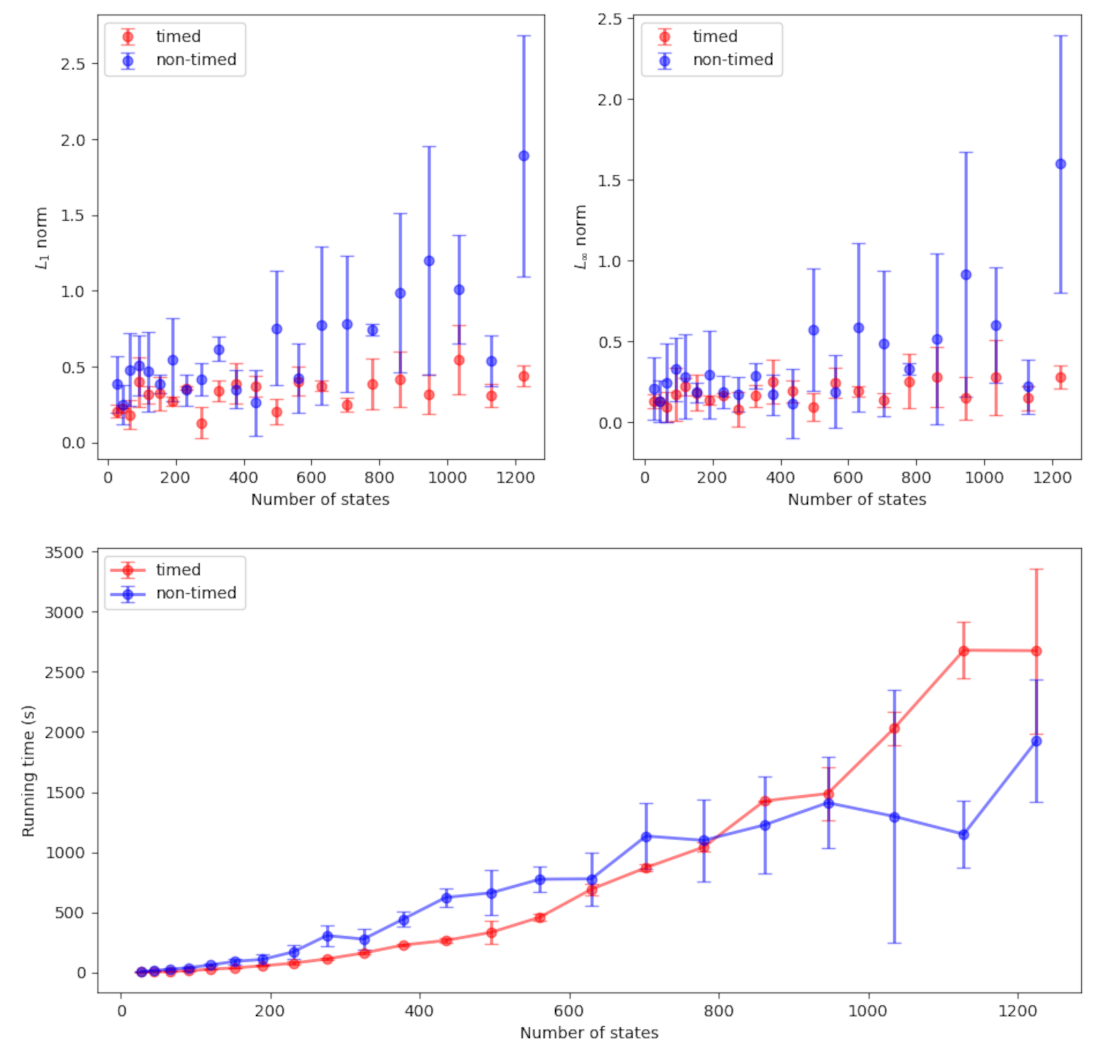}
    \caption{Comparison of the performance of the estimation for timed and non-timed observations on the tandem queueing network with different size of the queue.}
    \label{fig:tandem-res}
\end{figure}

Table \ref{tab:tandem-res} reports the results for some selected experiments. The second and third columns provide respectively the number of states and transitions of the parametric CTMC resulting from the choice of $\mathtt{c}$; the fourth column reports the average running time; while the fifth (resp.\ sixth) column details the average $L_1$-norm (resp. $L_\infty$-norm) of the vector $\delta = (\delta_i)$, calculated as $\norm{1}{\delta} = \sum_{i} |\delta_i|$ (resp. $\norm{\infty}{\delta} = \max_i |\delta_i|$). 

Fig.~\ref{fig:tandem-res} reports the results of all the experiments in a graphical format where measurements are presented together with their respective error bars. 

We observe that the running time is quadratic in the number of states (equivalently, linear in the size $|S|+|{\to}|$ of the model) both for timed and non-timed observations. However, for non-timed observations, the variance of the measured running times tends to grow with the size of the model. In this respect, we observed that large models required more iterations than small models to converge. Nevertheless, all experiments required at most 20 iterations.

As one may expect, the variance of the measured relative errors is larger on the experiments performed with non-timed observations, and the quality of the estimation is better when employing timed observations. 
Notably, for timed observations, the quality of the estimation remained stable despite the size of the model increased relatively to the size of the training set. This may be explained by the fact that the parameters occur in many transitions.

\section{Case Study: SIR modeling of pandemic} \label{sec:case-study}
In this section, we take as a case study the modeling pipeline proposed by Milazzo~\cite{Milazzo21} for the analysis and simulation in \prism\ of the spread of COVID-19 in presence of lockdown countermeasures.  
The modeling pipeline includes:
\begin{inlineenum}
\item parameter estimation from real data based on a modified SIR model described by means of a system of Ordinary Differential Equations; 
\item \label{item:translation} translation of the modified SIR model into a CTMC expressed as a \prism\ model; and 
\item \label{item:analaysis} stochastic simulation and model checking with \prism.
\end{inlineenum}

In particular, the \prism\ model devised in step~(\ref{item:translation}) is exactly the model depicted in Fig.~\ref{fig:bigSIR} (left). However, to perform the analysis, Milazzo had to apply ``a couple of modeling tricks (variable pruning and upper bounds) that allowed state space of the model constructed by \prism\ to be reduced by several orders of magnitude. The introduction of upper bounds to the values of variables actually introduces a small approximation in the model, that is negligible in practically relevant cases''~\cite{Milazzo21}. 
We argue that these kinds of modeling tricks are not uncommon in formal verification, but they require the modeler to ensure that the parameter values estimated for the original model are still valid in the approximated one. 

In this section, we showcase the use of our algorithm to semi-automatize this task. Specifically, we generate two training sets by simulating the SIR model in Fig.~\ref{fig:bigSIR} using \prism\ and, based on that, we re-estimate $\mathtt{beta}$, $\mathtt{gamma}$, and $\mathtt{plock}$ on an approximated version of the model (\cf\ Fig.~\ref{model:SIR-initial}) which is amenable to analysis in \prism.  

\begin{figure*}[ht]
\input{prismmodels/sir_init_clean.tex}
\caption{Approximated SIR model.}
\label{model:SIR-initial}
\end{figure*}

The first training set represents the spread of the disease without lockdown (i.e., $\mathtt{plock} = 1$), while the second one is obtained by fixing the value of $\mathtt{plock}$ estimated in~\cite{Milazzo21} (i.e., $\mathtt{plock} = 0.472081$). In line with the data set used in~\cite{Milazzo21}, both training sets consist of one timed observation reporting the number of infected individuals for a period of 30 days.


The estimation of the parameters $\mathtt{beta}$, $\mathtt{gamma}$ and $\mathtt{plock}$ is performed on the model depicted in Fig.~\ref{model:SIR-initial}.
As in~\cite{Milazzo21}, we use an approximated version of the original SIR model (\cf\ Fig.~\ref{fig:bigSIR}) obtained by employing a few modeling tricks: variable pruning, set upper bounds on the state variable $\mathtt{i}$, and re-scaling of the variable $\mathtt{r}$ in the interval $[0,\mathtt{nb\_r} -1]$. 
These modeling tricks have the effect to reduce the state space of the underlying CTMC, speeding-up in this way parameter estimation and the following model analysis. 


We perform the estimation in two steps. First, we estimate the values of $\mathtt{beta}$ and $\mathtt{gamma}$ on the first training set with $\mathtt{plock}$ set to $1$. 
Then, we estimate the value of $\mathtt{plock}$ on the second training set with $\mathtt{beta}$ and $\mathtt{gamma}$ set to the values estimated in the first step.

Each step was repeated 10 times by randomly re-sampling the initial values of each unknown parameter in the interval $[0,1]$. Table \ref{tab:sir-res} reports the average estimated values and absolute errors relative to each parameter. 
The average running time\footnote{Experiments were performed on a Linux machine with an AMD-Ryzen 9 3900X 12-Core processor and 32 GB of RAM.} of each execution of the algorithm was $89.94$ seconds.
\begin{table}[!ht]
\begin{center}
\rowcolors{2}{gray!10}{gray!40}
\begin{tabular}{|c|c|c|c|}
\hline
\textbf{Parameter} & \textbf{Expected Value} & \textbf{Estimated Value} & \textbf{Absolute Error}\\ \hline \hline
$\mathtt{beta}$     & $0.122128$   & $0.135541$    & $0.013413$\\ \hline
$\mathtt{gamma}$    & $0.127283$   & $0.128495$    & $0.001212$\\ \hline
$\mathtt{plock}$    & $0.472081$   & $0.437500$    & $0.034581$\\ \hline
\end{tabular}
\end{center}
\caption{Parameter estimation on the approximated SIR model.}
\label{tab:sir-res}
\end{table}

Our results confirm Milazzo's claim that the introduction of upper bounds to the values of state variables produces a small approximation in the model. 
Notably, we were able to achieve an accurate estimation of all the parameters from training sets consisting of  a single partially-observable execution of the original SIR model. As observed in Section~\ref{sec:experiments}, this may be due to the fact that each parameter occurs in many transitions. 

This case study demonstrates that our estimation procedure can be effectively used to simplify modeling pipelines that involve successive modifications of the model and the re-estimation of its parameter values. 
In line with the model checking problem, also our technique requires the modeler to take the size of the model into account. 

\section{Conclusion and Future Work} \label{sec:conclusion} 
We presented novel methods to estimate parameters values of CTMCs expressed as \prism\ models from timed and non-timed partially-observable executions. 
We demonstrated, through the use of a case-study, that our solution is a concrete aid in applications involving modeling and analysis, especially when the model under study requires successive adaptations which may lead to approximations that require re-estimation of the parameters of the model.

Notably, all the algorithms presented in this paper were devised following simple optimization principles borrowed from the MM optimization framework. 

We suggest that similar techniques can be employed to other modeling languages (e.g., Markov automata~\cite{EisentrautHZ10,EisentrautHZ10a}) and metric-based approximate minimization~\cite{BacciBLM17,BalleLPPR21}.
An interesting future direction of research consists in extending our techniques to non-deterministic stochastic models by integrating the active learning strategies presented in~\cite{BacciILR21}.


\bibliography{biblio}

\newpage
\appendix

\section{Missing proofs}
\begin{proof}[proof of Theorem~\ref{thm:gMinorizes}]
For convenience, we establish the result for a parametric CTMC $\mathcal{P}$ that satisfies the following assumptions: 
\begin{enumerate}[(A)]
    \item there is at most one transition between each pair of states; \label{item:A}
    \item for each transition $\rho = (s \trans{f_{\rho}} s')$, the map $f_\rho$ is either a constant (i.e., $f_\rho(\vec{x}) = c_\rho$ with $c_\rho \geq 0$) or of the form $f_\rho(\vec{x}) = c_\rho \prod_{i = 1}^{n} x_i^{a_{\rho i}}$ where $c_\rho > 0$ and $a_{\rho i} > 0$ for some $i \in \{1 \dots n\}$. \label{item:B}
\end{enumerate}
Assumption \ref{item:B} does not restrict the generality of the formulation. Indeed, a transition $\rho = (s \trans{f_{\rho}} s')$ where $f_\rho(\vec{x}) = \sum_{j = 1}^{m} c_j \prod_{i = 1}^{n} x_i^{a_{j i}}$ can be replaced by $m$ transitions of the form $s \trans{f_{i}} s'$ where $f_i(\vec{x}) = c_j \prod_{i = 1}^{n} x_i^{a_{j i}}$ ($j = 1 \dots m$). Note that in case $c_j = 0$ or $a_{j i} = 0$ for all $i \in \{1 \dots n\}$ the resulting map simplifies to a constant. 

As for Assumption~\ref{item:A}, assume $\mathcal{P}$ has $m \geq 2$ transitions from $s$ to $s'$, say $\rho_i = (s \trans{r_i} s')$ for $i = 1 \dots m$. We replace $s'$ with $m$ copies of it, say $s'_i$ for $i = 1 \dots m$ ---each having the same label and outgoing transitions of $s'$--- then redirect $\rho_i$ from $s'$ to $s'_i$ ---i.e., creating transitions $\bar{\rho}_i = (s \trans{r_i} s'_i)$. We then repeat this step until the chain has at most one transition between each pair of states. 

Note that the step described above ensures that the the source state $s$ does not chance its behavior, because by construction $s'$ is bisimilar to each $s'_i$ for $i = 1 \dots m$. As a consequence, the values $\alpha^j_s$ and $\beta^j_s$ are left unchanged and, for $i =1 \dots m$, $\alpha^j_{s'} = \alpha^j_{s'_i}$ and $\beta^j_{s'} = \beta^j_{s'_i}$. Therefore, by \eqref{eq:defGammaXi} we have that $\gamma^j_s$ is left unchanged and $\xi^j_{\rho_i} = \xi^j_{\bar{\rho}_i}$ for all $i = 1 \dots m$. 

Starting from the log-likelihood function, we proceed with the following minorization steps
\begin{align}
&\ln \mathcal{L}(\vec{x}) 
= \sum_{j = 1}^{J} \ln l(\vec{o}_j | \mathcal{P}(\vec{x})) 
= \sum_{j = 1}^{J} \ln \left( \sum_{s_{0:k_j}}  l(s_{0:k_j}, \vec{o}_j | \mathcal{P}(\vec{x}) ) \right) \\
&\geq \sum_{j = 1}^{J}  \sum_{s_{0:k_j}} \frac{l(s_{0:k_j}\vec{o}_j | \mathcal{P}(\hypX{m}))}{l(\vec{o}_j | \mathcal{P}(\hypX{m}))} \ln \left( \frac{l(\vec{o}_j | \mathcal{P}(\hypX{m}))}{l(s_{0:k_j}\vec{o}_j | \mathcal{P}(\hypX{m}))}  l(s_{0:k_j},\vec{o}_j | \mathcal{P}(\vec{x}))  \right) \tag{\eqref{eq:basicmin1}}\\
&\cong \sum_{j = 1}^{J}  \sum_{s_{0:k_j}} l(s_{0:k_j} | \vec{o}_j, \mathcal{P}(\hypX{m})) \ln \left( l(s_{0:k_j}\vec{o}_j | \mathcal{P}(\vec{x})) \right) \tag{up-to const} \\
&= \sum_{j = 1}^{J}  \sum_{s_{0:k_j}} l(s_{0:k_j} | \vec{o}_j, \mathcal{P}(\hypX{m})) \ln \left( \bool{\ell(s_{0:k_j}) {=} \ell^j_{0:k_j} }\prod_{t = 0}^{k_j-1} R(s_{t}, s_{t+1}) \cdot e^{- E(s_{t}) \tau^j_{t}} \right) \tag{\eqref{eq:LTObs2}}
\intertext{since $\ell(s_{0:k_j}) \neq\ell^j_{0:k_j}$ implies $l(s_{0:k_j} | \vec{o}_j, \mathcal{P}(\hypX{m})) = 0$, the above simplifies to}
&= \sum_{j = 1}^{J} \sum_{t = 0}^{k_j-1} \sum_{s_{0:k_j}} l(s_{0:k_j} | \vec{o}_j, \mathcal{P}(\hypX{m})) \big(\ln R(s_{t}, s_{t+1}) - E(s_{t}) \tau^j_{t} \big) \tag{*}\label{eq:gammaxi}
\end{align}
For $s \in S$ and $\rho = (s \trans{f_\rho} s')$, we define $\gamma^{j}_{s}(t)$ and $\xi^j_{\rho}(t)$ as
\begin{align*}
&\gamma^{j}_{s}(t) = l(S_t=s \mid \vec{o}_j, \mathcal{P}(\hypX{m})), 
&& \xi^{j}_{\rho}(t) = l(S_t=s, S_{t+1}=s' \mid \vec{o}_j, \mathcal{P}(\hypX{m})) \,.
\end{align*}
Then, we can reformulate \eqref{eq:gammaxi} as follows
\begin{align}
&\cong  \sum_{j = 1}^{J} \sum_{t = 0}^{k_j-1} \left( \sum_{\rho \in {\trans{\vec{x}}}} \xi^j_{\rho}(t) \ln f_{\rho}(\vec{x}) +
\sum_{s} \gamma^j_s(t) \tau^j_{t} (- E(s)) \right) \tag{up-to const}\\
&\cong  \sum_{\rho \in {\trans{\vec{x}}}} \xi_{\rho} \ln f_{\rho}(\vec{x}) +
\sum_{s} \gamma_s (- E(s)) \tag{rearrange}\\
&\cong \sum_{i = 1}^{n} \sum_{\rho \in {\trans{\vec{x}}}} \xi_{\rho} a_{\rho i} \ln x_i +
\sum_{s} \gamma_s (- E(s)) \tag{\eqref{item:B},up-to const}\\
&\geq \sum_{i = 1}^n \left[ 
    \sum_{\rho \in {\trans{\vec{x}}}} \xi_{\rho} a_{\rho i} \ln x_i 
    - \sum_{s} \sum_{\rho \in s \trans{\vec{x}} \cdot} \frac{f_\rho(\hypX{m}) a_{\rho i} \gamma_s \, }{a_\rho x_{m i}^{a_\rho}} x_i^{a_\rho} \right] \tag{**}\label{eq:laststep}\\
&= g(\vec{x} | \hypX{m}) \tag{by \eqref{eq:surrogate-timed}}
\end{align}
Where \eqref{eq:laststep} is justified by the following minorization of $- E(s)$ 
\begin{align*}
& - E(s) = \sum_{\rho \in s \to \cdot} - f_{\rho}(\vec{x}) \cong \sum_{\rho \in s \trans{\vec{x}} \cdot} c_{\rho} \left( - \prod_{i = 1}^{n} x_i^{a_{\rho i}} \right) \tag{up-to const, by \eqref{item:B}}\\
& \geq \sum_{\rho \in s \trans{\vec{x}}} - c_{\rho} \left(\prod_{i = 1}^{n} x_{m i}^{a_{\rho i}} \right) \sum_{i = 1}^{n} \frac{a_{\rho i}}{a_\rho} \left(\frac{x_i}{x_{m i}} \right)^{a_\rho} \tag{by \eqref{eq:basicmin3}} \\
& \geq - \sum_{i = 1}^{n} \sum_{\rho \in s \trans{\vec{x}}} \frac{f_{\rho}(\hypX{m}) a_{\rho i}}{a_\rho \, x_{m i}^{a_\rho}} x_i^{a_\rho} \tag{rearranging}
\end{align*}

As shown above, there exists a (non-negative) constant $c$ such that the surrogate function $g(\vec{x} | \hypX{m}) + c$ minorizes $\ln \mathcal{L}(\vec{x})$ at $\hypX{m}$.
\end{proof}

\begin{proof}[proof of Theorem~\ref{thm:hMinorizes}]
As done before, we establish the result for a parametric CTMC $\mathcal{P}$ satisfying assumption (\ref{item:A}) and (\ref{item:B}) form the proof of Theorem~\ref{thm:gMinorizes}.

Starting from the log-likelihood function $\ln \mathcal{L}(\vec{x})$, we proceed with the following minorization steps
\begin{align*}
&\ln \mathcal{L}(\vec{x}) = \sum_{j = 1}^{J} \ln l(\vec{o}_j | \mathcal{P}(\vec{x}))
= \sum_{j = 1}^{J}  \ln  \sum_{s_{0:k_j}} P[ s_{0:k_j}, \vec{o}_j | \mathcal{P}(\vec{x})]  \tag{by \eqref{eq:LObs1}} \\
&\geq \sum_{j = 1}^{J}  \sum_{s_{0:k_j}} {P[s_{0:k_j} | \vec{o}_j , \mathcal{P}(\hypX{m})]} \ln \left( \frac{P[ s_{0:k_j}, \vec{o}_j | \mathcal{P}(\vec{x})]}{P[s_{0:k_j} | \vec{o}_j , \mathcal{P}(\hypX{m})]} \right) \tag{by \eqref{eq:basicmin1}}\\
&\cong \sum_{j = 1}^{J}  \sum_{s_{0:k_j}} P[s_{0:k_j} | \vec{o}_j , \mathcal{P}(\hypX{m})] \ln P[ s_{0:k_j}, \vec{o}_j | \mathcal{P}(\vec{x})] \tag{up-to const}\\
&= \sum_{j = 1}^{J}  \sum_{s_{0:k_j}} P[s_{0:k_j} | \vec{o}_j , \mathcal{P}(\hypX{m})] \ln \left( \bool{\ell(s_{0:k_j}) = \ell_{0:k_j}} \prod_{t = 0}^{k_j-1} \frac{R(s_{t}, s_{t+1})}{E(s_{t})} \right) \tag{by \eqref{eq:LObs2}} \\
&= \sum_{j = 1}^{J}  \sum_{t = 1}^{k_j} \sum_{s_{0:k_j}} P[s_{0:k_j} | \vec{o}_j , \mathcal{P}(\hypX{m})] \big( \ln R(s_{t}, s_{t+1}) - \ln E(s_{t}) \big) \tag{$\triangle$} \label{eq:hatgammaxi}
\end{align*}
For $s \in S$ and $\rho = (s \trans{f_\rho} s')$, we define $\hat\gamma^{j}_{s}(t)$ and $\hat\xi^j_{\rho}(t)$ as
\begin{align*}
&\hat\gamma^{j}_{s}(t) = P[S_t=s \mid \vec{o}_j, \mathcal{P}(\hypX{m})], 
&& \hat\xi^{j}_{\rho}(t) = P[S_t=s, S_{t+1}=s' \mid \vec{o}_j, \mathcal{P}(\hypX{m})] \,.
\end{align*}
Then, we can reformulate \eqref{eq:hatgammaxi} as follows
\begin{align*}
&\cong \sum_{j = 1}^{J}  \sum_{t = 1}^{k_j} \left( \sum_{\rho \in {\trans{\vec{x}}}} \hat\xi^j_{\rho}(t) \ln f_\rho(\vec{x}) + \sum_{s} \hat\gamma^j_s(t) (- \ln E(s) ) \right)\tag{up-to const} \\
&= \sum_{\rho \in {\trans{\vec{x}}}} \hat\xi_{\rho} \ln f_\rho(\vec{x}) + \sum_{s} \hat\gamma_s (- \ln E(s) )\tag{rearrange} \\
&\geq \sum_{\rho \in {\trans{\vec{x}}}} \hat\xi_{\rho} \ln f_\rho(\vec{x}) + \sum_{s} \hat\gamma_s \left(1 - \ln E_{m}(s) - \frac{E(s)}{E_{m}(s)} \right) \tag{by \eqref{eq:basicmin2}}\\
&\cong  \sum_{i = 1}^{n} \sum_{\rho \in {\trans{\vec{x}}}} \hat\xi_{\rho} a_{\rho i} \ln x_i + \sum_{s} \hat\gamma_s \left( - \frac{E(s)}{E_{m}(s)} \right) \tag{\eqref{item:B}, up-to const} \\
& \geq \sum_{i = 1}^n \left[ 
    \sum_{\rho \in {\trans{\vec{x}}}} \hat{\xi}_{\rho} a_{\rho i} \ln x_i 
    - \sum_{s} \sum_{\rho \in s \trans{\vec{x}} \cdot} \frac{\hat{\gamma}_s f_\rho(\hypX{m}) a_{\rho i}}{E_{m}(s) a_\rho (x_{m i})^{a_\rho}} x_i^{a_\rho} \right] \tag{$\triangle\triangle$} \label{eq:hatminEs} \\
& = h(\vec{x} | \hypX{m}) \tag{by \eqref{eq:surrogate-nontimed}}
\end{align*}
Where the step \eqref{eq:hatminEs} is justified by the minorization of $- E(s)$ provided in the proof of Theorem~\ref{thm:gMinorizes}.

Therefore, there exists a constant $c$ such that the surrogate function $h(\vec{x} | \hypX{m}) + c$ minorizes $\ln \mathcal{L}(\vec{x})$ at $\hypX{m}$. 
\end{proof}

\end{document}